\documentclass[letterpaper]{article} % DO NOT CHANGE THIS
\usepackage{aaai2026}  % DO NOT CHANGE THIS
\usepackage{times}  % DO NOT CHANGE THIS
\usepackage{helvet}  % DO NOT CHANGE THIS
\usepackage{courier}  % DO NOT CHANGE THIS
\usepackage[hyphens]{url}  % DO NOT CHANGE THIS
\usepackage{graphicx} % DO NOT CHANGE THIS
\usepackage{natbib}  % DO NOT CHANGE THIS AND DO NOT ADD ANY OPTIONS TO IT
\usepackage{caption} % DO NOT CHANGE THIS AND DO NOT ADD ANY OPTIONS TO IT
\usepackage{amsmath, amssymb, amsthm}
\usepackage{amsfonts}
\usepackage{dsfont}
\usepackage{soul}
\usepackage{newfloat}
\usepackage{listings}
\usepackage{xcolor}
\usepackage{booktabs} 
\usepackage[table]{xcolor}

\usepackage{algorithm}
\usepackage[noend]{algorithmic}
\usepackage{newfloat}
\usepackage{listings}

\DeclareCaptionStyle{ruled}{labelfont=normalfont,labelsep=colon,strut=off} % DO NOT CHANGE THIS
\floatstyle{ruled}
\newfloat{listing}{tb}{lst}{}
\floatname{listing}{Listing}
\newtheorem{theorem}{Theorem}
\newtheorem{lemma}[theorem]{Lemma}
\newtheorem{corollary}[theorem]{Corollary}

\newtheorem{definition}{\textbf{Definition}}

\title{On the Probabilistic Learnability of Compact
Neural Network Preimage Bounds}
\author{
    Luca Marzari, Manuele Bicego, Ferdinando Cicalese and Alessandro Farinelli}
\affiliations{
    Department of Computer Science, University of Verona, Italy\\
    Contact author: \textit{luca.marzari@univr.it}
}

\usepackage{bibentry}
\begin{document}

\maketitle

\begin{abstract}
Although recent provable methods have been developed to compute preimage bounds for neural networks, their scalability is fundamentally limited by the \#P-hardness of the problem. In this work, we adopt a novel probabilistic perspective, aiming to deliver solutions with high-confidence guarantees and bounded error. To this end, we investigate the potential of bootstrap-based and randomized approaches that are capable of capturing complex patterns in high-dimensional spaces, including input regions where a given output property holds. 
In detail, we introduce \textbf{R}andom \textbf{F}orest \textbf{Pro}perty \textbf{Ve}rifier (\texttt{RF-ProVe}), a method that exploits an ensemble of randomized decision trees to generate candidate input regions satisfying a desired output property and refines them through active resampling. Our theoretical derivations offer formal statistical guarantees on region purity and global coverage, providing a practical, scalable solution for computing compact preimage approximations in cases where exact solvers fail to scale.
\end{abstract}

% Uncomment the following to link to your code, datasets, an extended version or similar.
% You must keep this block between (not within) the abstract and the main body of the paper.
\begin{links}
    \link{Code}{https://github.com/lmarza/ProbVerNet}
    %\link{Extended version}{https://aaai.org/example}
\end{links}

\section{Introduction}
\label{sec:introduction}

The ability of Deep neural networks (DNNs) to learn complex patterns from vast amounts of data has allowed them to tackle challenging tasks in several domains \citep{image,manipulation2,eps_retrain}. However, as DNNs become more powerful and pervasive, safety concerns have become increasingly prominent. In particular, DNNs are often considered "black-box" systems, meaning their internal representation is not fully transparent. A crucial weakness of DNNs is the vulnerability to adversarial attacks \citep{adversarial, amir2022verifying}, wherein small, imperceptible modifications to input data can lead to wrong and potentially catastrophic decisions when deployed. 

On top of standard \textsc{DNN-Verification} \citep{LiuSurvey,crown,acrown,bcrown, wei2024modelverification}, which aims to establish provable guarantees that the network adheres to specific formal specifications,  
recent works \citep{CountingProVe,kotha2023provably,zhang2024provable}, based on seminal results of \cite{dathathri2019inverse,matoba2020exact}, have formalized
the quantitative version of the verification problem, namely identifying the subset of a desired input region where a DNN produces (or not) a desired output. This problem is formally defined as \textsc{AllDNN-Verification} or provable DNNs' preimage bound computation.\footnote{We note that \citet{CountingProVe} and \citet{kotha2023provably} independently and contemporaneously addressed the same underlying problem under different names. In this work, we use \textsc{AllDNN-Verification} problem or bounding the DNN’s preimage interchangeably.}  
Computing the preimage bound provides a more informative and fine-grained characterization of the model’s behavior, enabling the quantification and localization of the full region of inputs that lead to unsafe outputs, rather than relying on the mere existence of (possibly) isolated counterexamples. This information can be used to guide model debugging, improve training procedures through targeted data augmentation, and inform safe recovery strategies by identifying and avoiding risky regions during deployment. In this context, producing compact representations of such unsafe regions is crucial to enhance explainability and support safer fallback mechanisms, as compact regions are easier to interpret.

However, as for most of the classical enumeration problems (e.g., \textsc{AllSAT} \citep{valiant1979complexity}), the exact enumeration of neural network preimage bounds is computationally prohibitive, as the problem has been shown to be \#P-hard \citep{CountingProVe}. 
To circumvent such a problem, recent efforts \cite{zhang2024provable, zhang2025premap} have explored the combination of sound under- and over-approximations to approximate the preimage bounds of a neural network with a set of polytopes as compact as possible. Nonetheless, these solutions still face significant scalability issues due to the reliance on a provably sound solution. We argue that the \#P-hardness of the problem and its intractability necessitate novel probabilistic solutions that balance computational feasibility with accuracy. Specifically, in this work, we investigate an approximate variant of the \textsc{AllDNN-Verification} problem which is probabilistically solvable, that is, we devise an efficient algorithm that delivers an  {\em approximate and compact} solution with high-confidence guarantees and bounded error. In a similar vein, \citep{eProve} proposes a probabilistic enumeration of preimage bounds. However, their focus lies primarily on maximizing coverage, rather than on ensuring compactness of the solution. In fact, their reliance on a single decision tree to provide the solution often results in the generation of a large number of polytopes, which in complex scenarios can even exhaust memory resources, producing highly fragmented representations that are difficult to interpret and impractical for downstream tasks such as safe recovery or explanation.
In contrast, in this work, we explore the potential of bootstrap-based and randomized approaches that are capable of capturing complex patterns in high-dimensional spaces, including input regions where a given output property holds. Our probabilistic bounds are from the realm of  \textit{statistical prediction on tolerance limits} \citep{wilks}, which enable high-confidence guarantees on region purity and global coverage.

Specifically, we present \textbf{R}andom \textbf{F}orest-\textbf{Pro}perty \textbf{Ve}rifier (\texttt{RF-ProVe}), a novel probabilistic approach based on a random forest-inspired classifier. In detail, we exploit an ensemble of randomized decision trees structurally similar to a random forest, but without relying on the traditional majority voting scheme for classification \cite{breiman2001random}.\footnote{Throughout the paper, we slightly abuse notation by referring to this ensemble as a random forest, even though it does not employ majority voting.} This choice is motivated by the goal of representing the preimage bounds of a neural network as axis-aligned boxes. Alternative representations, such as unions of halfspaces, are computationally more complex and often less interpretable \cite{blumer1989learnability}. Although random forests implicitly partition the input space into axis-aligned regions, they are not represented in an explicit way. To address this, we extract axis-aligned boxes directly from the decision paths leading to the leaves of the trees. However, while these leaf regions may appear pure (e.g., according to the Gini index), their reliability could be compromised by limited training data. To mitigate this, we employ a filtering phase based on an \textit{active resampling strategy} that validates the purity of each region. Crucially, our probabilistic guarantees, based on \citet{wilks} results, allow us to formally determine the number of resampling points needed during this filtering phase. This enables us to return a final set of regions for which we can provide high-confidence guarantees on both their individual purity and the overall coverage of the preimage.

Our empirical evaluation on standard verification benchmarks demonstrates that \texttt{RF-ProVe} provides a valuable probabilistic framework for challenging instances that are difficult to verify with exact or provable solvers, producing compact solutions with fewer polytopes compared to existing approaches for the 
(approximate) \textsc{AllDNN-Verification} problem.

In summary, the contributions of this paper are:
\begin{itemize}
    % \item We introduce a novel probabilistic perspective on the \textsc{AllDNN-Verification} problem.
    \item We present \texttt{RF-ProVe}, a random forest-based method that combines passive learning with an active resampling strategy to efficiently approximate unions of axis-aligned boxes representing compact neural network preimages.
    \item We develop probabilistic bounds based on \citet{wilks} statistical tolerance limits, providing high-confidence assurances on the purity and coverage of the extracted input regions, guaranteeing a scalable and practical approximate solution to the (\#P-hard)  exact verification problem.
\end{itemize}

%======= PRELIMINARIES

\section{Preliminaries and Related Work}
In this section, we provide the reader with all the necessary basic definitions and notation on \textsc{allDNN-Verification} to easily follow the paper. Moreover, we discuss related work on the problem we aim to address.

Consider a deep neural network $f:\mathbb{R}^N\to\mathbb{R}$ and a safety property $\mathcal{P} = \langle\mathcal{X}, \mathcal{Y}\rangle$ to be verified. In detail, a safety property encodes an input-output relationships for $f$ and it is composed of a precondition on the input $\mathcal{X} \subset \mathbb{R}^N$, that identifies a portion of the input space where we want a specific postcondition $\mathcal{Y}$ to be satisfied on the output of $f$. Without loss of generality, in the following, we assume that the DNNs we verify have a single output node, i.e., performing a binary classification. One can simply enforce this condition for networks that do not satisfy this assumption by adding one layer and encoding the requirements of $\mathcal{Y}$ in a single output node as a margin between logits, which is positive if only if the property is respected \citep{LiuSurvey,bcrown}.

\subsection{ALLDNN-Verification or DNN's Preimages Bounds Computation}

\begin{figure}[b]
    \centering
    \includegraphics[width=\linewidth]{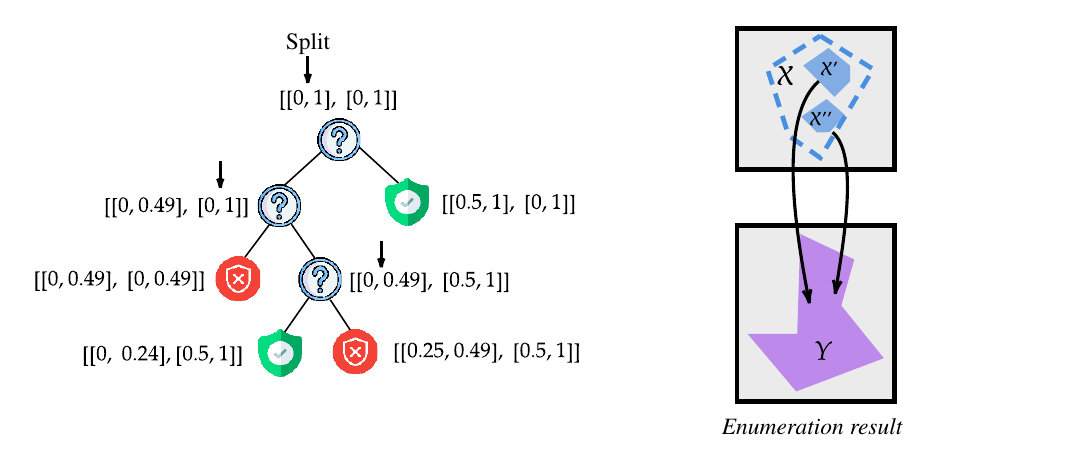}
    \caption{Illustrative overview of \textsc{AllDNN-Verification} problem.}
    \label{fig:enum}
\end{figure}

The \textsc{AllDNN-Verification} problem \cite{CountingProVe}, also referred to as exact preimage bounds of a neural network \cite{matoba2020exact, kotha2023provably, zhang2024provable}, asks for the subset of points in the input space ${\cal X}$ that a given function $f$ maps to a given subset ${\cal Y}$ of output values, i.e., the pre-image of ${\cal Y}$ with respect to $f.$

\begin{definition}[AllDNN-Verification Problem]\label{def:alldnn_ver}
\phantom{abc}

{\bf Input}: A tuple $\mathcal{T}=\langle f, \mathcal{X},\mathcal{Y}\rangle$. 

{\bf Output}: $\Gamma(\mathcal{T}) = \Big\{ x \in  \mathcal{X} \mid f(x) \in \mathcal{Y}\Big\}$.
\end{definition}

%{\color{red} 
For the sake of simplifying the presentation, we focus on a binary classification task, and we assume that $f$ is the boolean function obtained by thresholding the single output of a DNN, i.e., such that $f(x) = 1$ iff  the output of the DNN is $\geq 0.5,$ hence we have ${\cal Y} = \{1\}$ and 
$\Gamma(\mathcal{T}) = \left\{ x \in  \mathcal{X} \mid f(x) = 1.\right\}$
%}

One possible approach to solve this challenge in an exact fashion, e.g., discovering the set of polytopes that exactly cover the volume of $\Gamma(\mathcal{T})$, $Vol(\Gamma(\mathcal{T}))$, is to leverage the branch-and-bound (BaB) \cite{BaB} process commonly used in verification and recursively record which regions are (or are not) correctly mapped into $\mathcal{Y}$, as illustrated in Fig. \ref{fig:enum}. 
However, as shown in \cite{CountingProVe}, similarly to standard verification, the number of splits either on the input or on the network's non-linearities required in the worst case can grow exponentially, since the problem is \#P-hard.  
Recent progress has been made through \textit{linear relaxation} techniques \cite{crown, acrown, bcrown, autolirpa}, which over-approximate the network’s non-linear behavior and enable backward analysis to compute conservative estimates of the preimage. However, approaches like \cite{kotha2023provably} rely on sound over-approximations and still face scalability limitations, making them unsuitable for quantitative verification. To address such an issue, novel solutions have been proposed in \cite{zhang2024provable, zhang2025premap, bjorklund2025efficient} for the approximate version of the problem:

\begin{definition}[Approximate AllDNN-Verification]\label{def:approx_alldnn}
\phantom{abc}

{\bf Input}: $\mathcal{T}=\langle f, \mathcal{X},\mathcal{Y}\rangle, c \in (0,1]$. 

{\bf Output}: a set ${\cal B}= \{b_1, \dots, b_m\}$ of disjoint polytopes such that $\bigcup_i b_i \subseteq \Gamma({\cal T})$ and 
$\frac{Vol(\bigcup_i b_i)}{Vol(\Gamma(\mathcal{T}))} \geq c$.
\end{definition}

In this setting, the input includes the tuple $\mathcal{T}$ and a desired coverage ratio (c) of the volume of the preimage set $\Gamma(\mathcal{T})$. Since computing this volume in an exact fashion is computationally prohibitive, typically an estimation is computed, for example, using the Monte Carlo method obtaining $Vol(\Gamma(\mathcal{T})) = Vol(\mathcal{X}) \times \frac{1}{k} \sum_{i=1}^k \mathds{1}_{f(x_i) = 1}$
where $x_1, \ldots, x_k$ are sampled from the input domain $\mathcal{X}$, and $\mathds{1}_{f(x_i) = 1}$ indicates whether each sample is mapped to the target set $\mathcal{Y}$ encoded in $\mathcal{T}$.
The goal then is to construct a set $\mathcal{B}$ of disjoint polytopes (e.g., axis-aligned hyperrectangles) that under-approximate $\Gamma(\mathcal{T})$ while covering at least a fraction $c \in (0, 1]$ of the estimated volume $Vol(\Gamma(\mathcal{T}))$. Specifically, \cite{zhang2024provable,bjorklund2025efficient,zhang2025premap} extend the work of \cite{kotha2023provably} by introducing a novel combination of sound under- and over-approximation strategies based on neural network linearization, effectively guiding the divide-and-conquer procedure for estimating the preimage bounds set. 
Nonetheless, these approaches are deterministic and sound, but due to the absence of a theoretical bound, to guarantee a desired approximation, the algorithm needs to empirically verify it at run time by estimating the coverage via sampling, which can still lead to scalability issues, as shown also in our experiments.

In this work, we focus on a novel probabilistic relaxation of the problem, where the solution is allowed to (possibly) include some incorrect input points but guaranteeing that with confidence at least $1-\delta$ the volume of the incorrect points is bounded to at most an $\epsilon$-fraction of the returned solution, and, moreover,  this covers at least a desired portion of the exact preimage set.

\begin{definition}[Probabilistic Approximate AllDNN-Verification]\label{def:prob_approx_alldnn}
\phantom{abc}

{\bf Input}: $\mathcal{T}, c \in (0,1] \text{ and } \epsilon, \delta \in (0,1)$.

{\bf Output}: A set ${\cal B} = \{b_1, \dots, b_m\}$ of  polytopes such that, with probability at least $1-\delta$,
\[
\frac{Vol( \Gamma(\mathcal{T}) \cap \bigcup_i b_i)}{Vol(\Gamma(\mathcal{T}))} \geq c \quad (coverage)
\] 
and
\[\frac{Vol(\left\{f(x) \notin \mathcal{Y} \mid x \in \bigcup_i b_i\right\})}{Vol(\{ \bigcup_i b_i \})}
\leq \epsilon \quad (error).
\]

\end{definition}

In this vein, \cite{eProve} employs a sampling-based approach to generate probabilistically sound reachable sets and designs efficient heuristics to support the BaB verification process, ultimately collecting a set of axis-aligned hyperrectangles. However, as noted earlier, their reliance on a single decision tree often results in highly fragmented representations of the preimage bounds and, in the worst-case scenarios, can lead to memory exhaustion. In this work, we use both approaches, namely the sound under-approximation provided by \cite{zhang2025premap} and the probabilistic one provided by \cite{eProve}, as baselines for our empirical evaluation.

%===== RF-PROVE

\section{RF-ProVe: a Novel Probabilistic Approach}

While recent approximate solutions for \textsc{AllDNN-Verification} have made significant progress in efficiently addressing the problem, they often face trade-offs between scalability and provable coverage guarantees. 
To address this, we propose \texttt{RF-ProVe}, a novel probabilistic random forest learning-inspired method specifically tailored for the probabilistic \textsc{AllDNN-Verification} problem.

\begin{figure}[b]
    \centering
    \includegraphics[width=0.9\linewidth]{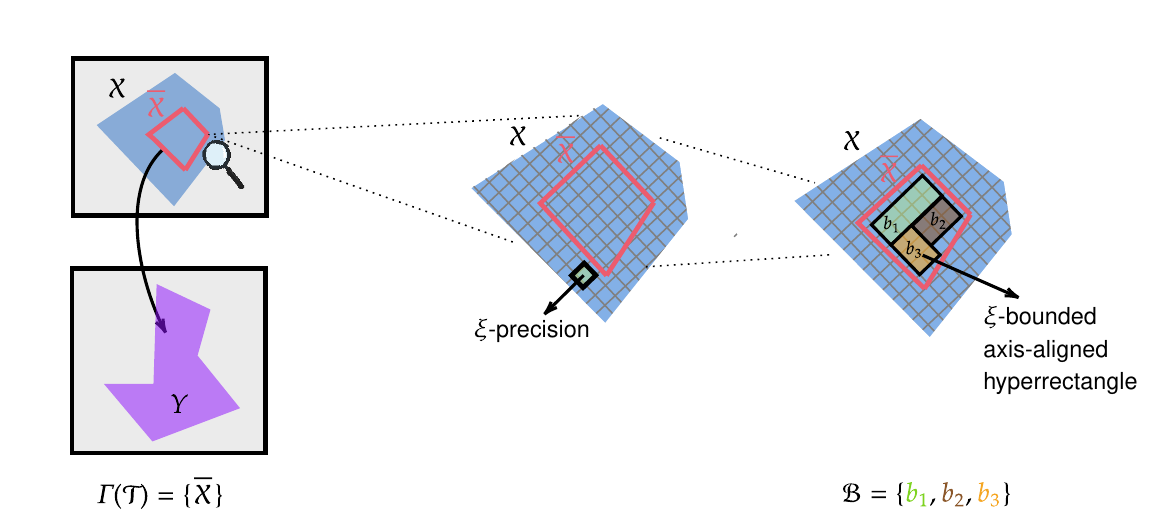}
    \caption{Explanatory image of the solution returned by our \texttt{RF-ProVe}.}
    \label{fig:rf_prove_sol}
\end{figure}

Our key idea is to leverage the potential of bootstrap and randomized-based approaches, which are well-suited for capturing complex patterns in high-dimensional spaces.
Fig.\ref{fig:rf_prove_sol} illustrates the overall problem and our proposed approach. Given a target output property $\mathcal{Y}$, our objective is to identify the corresponding region(s) in the input space, denoted as $\overline{\mathcal{X}}$, that the neural network maps into $\mathcal{Y}$.
Since the location of such input regions is not known a priori, we propose to sample labeled examples from the original input space $\mathcal{X}$ and use them to guide the construction of a collection of decision trees. In detail, these trees are used to partition $\mathcal{X}$ into subregions up to a fixed depth $D$,  which inherently defines a user-defined precision parameter $\xi = 2^{-D}$. Consequently, our goal becomes identifying, with high confidence and bounded error, a collection $\mathcal{B}$ of $\xi$-bounded axis-aligned boxes that approximate, as tightly as possible, the neural network preimage of $\mathcal{Y}$. We highlight that the discretization step does not compromise the soundness of the procedure, as the input space can be assumed to be discretized up to the resolution allowed by machine precision. Moreover, if a region cannot be resolved to the required $\xi$-precision, it is excluded from the returned set, which preserves the correctness of the final result. In fact, in the worst case, this may lead to a conservative approximation, i.e., a looser under-approximation of the true preimage bounds.
Importantly, our method leverages \textit{statistical prediction via tolerance limits} \cite{wilks} to derive novel theoretical guarantees for the use of randomized ensemble learners such as random forests on both the error within individual regions and the overall coverage of the returned set of boxes. 

\begin{algorithm}[t]
\caption{\texttt{RF-ProVe}}\label{alg:RF-ProVe}
\begin{algorithmic}[1]
\small
\STATE \textbf{Input:} $\mathcal{T} =\langle f,\mathcal{X}, \mathcal{Y}\rangle$, %\xi$ desired precision, 
$T$ \# decision trees, $D$ maximum depth,
$R$ leaf purity desired, $\delta$ confidence error, $m$ \# training examples, $k$ testing examples, $c$ desired coverage.
\STATE \textbf{Output: } $\mathcal{B}$ set of regions (hyperrectangles) satisfying $\mathcal{Y}$, estimated coverage reached.
\vspace{0.2cm}

\STATE  $\mathcal{B} \gets \emptyset$
\STATE S $\gets \texttt{GetExamples}(f, m, \mathcal{X}, \mathcal{Y})$
\STATE rf $\gets \texttt{RandomForest}$(S, T, D)
\FOR{tree in rf.trees}   
    \STATE $B\gets \texttt{GetPurePositiveLeaves}$(tree, $\mathcal{Y}$)
    \STATE $n = \frac{\ln(\delta)}{\ln(R)}$
     \STATE $\hfill \rhd$ filtering phase.
    \FOR{$b$ in $B$} 
    \IF{$\texttt{SamplesInside}(b) \geq n$}
        \STATE $\mathcal{B} \gets \mathcal{B}\;\cup\;b$
    \ELSE
    \STATE $S' \gets \texttt{GetExamples} (f, n , b, \mathcal{Y})$ 
        \IF{$f(x_i) =1  \;\forall\;x_i \in S'$}
            \STATE $\mathcal{B} \gets \mathcal{B}\;\cup\;b$
        \ENDIF
    \ENDIF
    \ENDFOR
    \STATE $\mathcal{B} \gets \texttt{RemoveDuplicateBoxes}(\mathcal{B})$
    \STATE coverage, k $\gets \texttt{EstimateCoverage}(\mathcal{B},k)$
    \IF{coverage $\geq$ c}
        \STATE break
    \ENDIF
\ENDFOR

\STATE \textbf{return} $\mathcal{B}$, coverage
\end{algorithmic}

\end{algorithm}

\subsubsection{Random Forest Classifier}

The first component of our novel probabilistic approach is a random forest-inspired classifier \cite{breiman2001random}. Given a labeled dataset $S = \{(x_i, y_i)\}_{i=1}^m$, where $x_i \in \mathbb{R}^N$ and $y_i \in \{0, 1\}$, we train a random forest with $T$ (fixed) decision trees (lines 4-5). Each tree creates a partition of the input space into axis-aligned boxes, corresponding to its leaf nodes up to a maximum predefined depth $D$ to be reached in each tree. We use the Gini criterion to maximize the purity of leaves (i.e., maximizing the probability of having leaves containing only positive or non-positive examples from $S$). Hence, after the training of the classifier, we collect all pure positive leaves (boxes containing only positive examples in $S$) across the $T$ trees and store them in $B$ (lines 6-7).

\subsubsection{Active Resampling Strategy}

Each box in the set $B$, denoted $b_i \subseteq \mathbb{R}^N$, is an axis-aligned hyperrectangle representing a candidate preimage region in the input space. These boxes are initially extracted from leaves of decision trees in the random forest that appear pure with respect to the target output property $\mathcal{Y}$, based on the Gini. However, this criterion may overestimate the true purity of a region, especially when leaves contain only a few training samples. As a result, a region may appear purely positive due to sampling bias, despite containing unobserved non-positive points. To mitigate this issue and obtain stronger probabilistic guarantees, we introduce an active resampling strategy (lines 8–19). Specifically, we compute the number of positive samples $n = \frac{\ln(\delta)}{\ln(R)}$ derived from our theoretical analysis (detailed in the next paragraph), that each candidate box $b_i$ should contain in order to be stored in the returned solution. Hence, we first verify whether a positive leaf, i.e., a $b_i$, already contains at least $n$ such samples; if it does, we include $b_i$ in $\mathcal{B}$. Otherwise, we uniformly sample $n$ new inputs from $b_i$, label them using the neural network $f$, and collect the results in a set $S'$. If all inputs $x_i \in S'$ satisfy $f(x_i) = 1$, then $b_i$ is added to $\mathcal{B}$; otherwise, it is discarded. As we will show in the next paragraph, this procedure guarantees that, with probability at least $1 - \delta$, each accepted box $b_i \in \mathcal{B}$ contains at least a fraction $R$ of its volume classified as positive. 
The boxes in $\mathcal{B}$ may partially overlap, as only full containment is eliminated by the filtering step (line 20). Notwithstanding the theoretical guarantee on the achieved coverage (Theorem \ref{theorem:coverage}), since this, in practice, may speed up the convergence, we also estimate the volume of the coverage of the current solution using a Monte Carlo estimation as in \cite{zhang2025premap} (line 21).

Specifically, we count how many new examples in a fresh test set of $k$ samples fall within at least one of the collected boxes in $\mathcal{B}$, i.e., satisfying $f(x)=1$. This empirical estimate serves as a proxy for the true volume of the positive part of the preimage under construction. If the estimated volume reached the desired coverage ratio, we stop the loop and return the solution $\mathcal{B}$ and the corresponding coverage; otherwise, we proceed (lines 22-26).

\subsubsection{Theoretical Guarantees}

In this part, we discuss the theoretical guarantees underlying our \texttt{RF-ProVe} approach. 
To this end, we begin by revisiting the key result on \textit{statistical prediction of tolerance limits} \cite{wilks}, adapting it to our specific setting.

\begin{lemma}[\citep{wilks}]\label{lemma_Wilks}
    Fix a function $g: \mathbb{R}^d \mapsto \mathbb{R}.$ For any  $R \in (0,1)$ and integer $n$, given a sample $X_1$ of $n$ values from a (continuous) set $X \subseteq \mathbb{R}^d$ the probability that for at least a fraction $R$ of the values in a further possibly infinite sequence of samples $x$ from  $X$ the value of $g(x)$ is not smaller (respectively larger) than the minimum value $\min_{x \in X_1} g(x)$ (resp.\ maximum value $\max_{x \in X_1} g(x)$) of $g$ estimated with the first $n$ samples is at least equal to $1-\delta$, where $\delta$ is the value satisfying the following equation
    \begin{equation} \label{eq:Wilks}
        1- \delta  = n \cdot \int_R^1 x^{n-1}\;dx  = (1 - R^n)
    \end{equation}
    
\end{lemma}

\begin{corollary}\label{lemma:wilks}
Let $g\colon \mathbb{R}^N \to \mathbb [0,1]$ be a real‐valued function and let $\mathcal X\subseteq\mathbb{R}^N$ be a region of interest. 
Let $f$ be the function mapping points from $\mathbb{R}^N$ to $\{0,1\}$ defined by $f(x) = 1$ iff $g(x) \geq 1/2.$ Fix  $\delta, R\in(0,1)$ and let $n \geq \frac{\ln \delta}{\ln R}.$

Draw $n$ i.i.d.\ samples $x_1,\dots,x_n$ from $\mathcal X.$ 
Let $p = \frac{Vol(\{x \in {\cal X} \mid f(x) = 1)\})}{Vol({\cal X})},$ be the true fraction of points in $\mathcal{X}$ which are positive for $f$. 
If for each $i = 1, \dots, n$ we have $f(x_i)=1$ then $$\Pr\bigl[p < R\bigr] < \delta.$$

Equivalently, with probability at least $1-\delta$ the region $\mathcal{X}$ has at least a fraction $p\geq R$ of positive points for $f$.
\end{corollary}

Importantly, Lemma \ref{lemma_Wilks} and Corollary \ref{lemma:wilks}  do not require any knowledge of the probability distribution governing the function of interest and thus also apply to general DNNs.

\begin{definition}[$\xi$-bounded hyperrectangle]
A rectilinear $\xi$-bounded hyperrectangle is defined as the cartesian product of intervals of size at least $\xi.$
Moreover, for $\xi > 0,$  we say that a rectilinear hyperrectangle $r = \times_i [\ell_i, u_i]$ is {\em $\xi$-aligned}
if for each $i,$ both extremes $\ell_i$ and $u_i$ are  multiples of $\xi.$ 
\end{definition}

\begin{lemma}[Positive Samples in $b^{(\xi)}$] \label{chernoff-wilks}
Let $\mathcal{X} \in \mathbb{R}^N$ be a region of interest. Fix $\xi, R, \delta \in (0,1)$ and let $n = \frac{\ln \delta}{\ln R}$ be the sample size sufficient to guarantee the bound in Lemma \ref{lemma:wilks}.
Let $Vol_{\xi} = \xi^N$ be the volume of a hyperrectangle where each side is of size $\xi.$ 
Fix $\alpha > 1$ and let $m  > \frac{n \alpha}{Vol_{\xi}},$ 
$\mu = m \cdot Vol_{\xi}.$
and $P_{\neg} = exp(-\frac{(1-\frac{1}{\alpha})^2 \mu}{2}).$ 
Consider a hyperrectangle $b^{(\xi)} \subseteq \mathcal{X}^N$ of volume $Vol_{\xi}.$
Then, the probability that among $m$ points independently and uniformly sampled from the input space ${\cal X}$ 
 less than $n$ points are from $b^{\xi}$ is 
$\leq P_{\neg}.$
\end{lemma}
\begin{proof}
For $i = 1, \dots, m,$ let $X_i$ be the indicator random variable of the event that the $i$th point is from $b^{(\xi)}.$
Then, we have $\mathbb{E}[X_i] = Vol_{\xi}$ and $\mu = m \mathbb{E}[X_i] = \mathbb{E}[\sum_i X_i].$ Then, 
the desired result is a direct consequence of the Chernoff bound \citep{mitzenmacher2005probability}.
\end{proof}

\begin{theorem}[Coverage Guarantees of \texttt{RF-ProVe}] \label{theorem:coverage}
Let $\mathcal{X} \in \mathbb{R}^N$ be a region of interest.
Let $\mathcal{B} = \{b_1, \dots, b_k\}$ be the collection of disjoint hyperrectangles containing all and only the input positive points of the neural network for $\mathcal{X}$, i.e., $\mathcal{B} = \cup_j b_j = f^{-1}(1),$ where $f$ is the function computed by the neural network. Assume that for each $j = 1, \dots, k,$ it holds that 
$b_j$ is $k \xi$-bounded, for some $k \geq 3,$  hence, in particular, we have $Vol(b_j) \geq k^N Vol_{\xi}.$ Let ${\cal B^*} = \bigcup_j  b_j$ be the total exact preimage bound.

Consider a random forest with $T$ random trees trained on $m$ samples, with $m$, satisfying the bound of Lemma \ref{chernoff-wilks}. 
Let ${\cal B}^A = \{b_1^A, b_2^A, \dots, b_s^A\}$, be the set of (possibly overlapping) hyperrectangles that estimate the preimage output bounds computed by \texttt{RF-ProVe}.
Then, we have that $(\frac{k-2}{k})^N Vol({\cal B}^*) \leq Vol({\cal B}^A  \cap{\cal B}^*)$  and  $Vol({\cal B}^A \cap {\cal B}^*) \geq R \;Vol(\mathcal{B}^A)$. In particular, the fraction of incorrect points (false positives) among the output boxes satisfies: $Vol(\mathcal{B}^A \setminus \mathcal{B}^*) \leq (1-R) \;Vol(\mathcal{B}^A)$.

\end{theorem}
\begin{proof}

Recall the definitions and the notation of Lemma \ref{chernoff-wilks}. For the sake of simplifying the argument,
We will use the following lemma from \cite{eProve}, rephrased in the context of our present setting. 

\begin{lemma}\label{lemma3} \cite{eProve}
Fix a real number $\xi > 0$ and an integer $k \geq 3.$ For any $\gamma > k \xi$ and any $\gamma$-bounded rectilinear hyperrectangle 
$r \subseteq \mathbb{R}^N,$ 
there is an $\xi$-aligned rectilinear hyperrectangle $r^{(\xi)}$ such that: (i) $r^{(\xi)} \subseteq r$; and 
(ii) $Vol(r^{(\xi)}) \geq \left( \frac{k-2}{k}\right)^N Vol(r).$
\end{lemma}

By applying this lemma to each hyperrectangle $b_j$ we obtain a collection of rectilinear $\xi$-bounded and $\xi$-aligned
hyperrectangles $\hat{b_1}, \dots, \hat{b_k},$ such that for each $j=1, \dots, k,$ we have $\hat{b}_j \subseteq b_j,$ and 
$Vol(\hat{b}_j) \geq (\frac{k-2}{k})^N Vol(b_j).$ 
Let $\hat{\cal B} = \bigcup_j \hat{b}_j.$ For each $j$ and each $\xi$-aligned hyperrectangle $b^{(\xi)}$ of volume $\xi^N$ 
contained in $\hat{b}_j$ we have that the probability that for each tree the training set used for building the forest $T$ contains less than $n$ points sampled from $b^{(\xi)}$ is at most $P_{\neg}^T.$ 
Let $P_{\neg, {\cal B}}$ be the probability that for some $j \in [k]$ there is an $\xi$-aligned hyperrectangle of volume $\xi^N$ included in 
$\hat{b}_j$ such that in the training set of each tree, less than $n$ samples are from $b^{(\xi)}.$ Then, by the union bound, we have $P_{\neg, \hat{\cal B}} \leq \frac{Vol(\hat{\cal B})}{\xi^N} P_{\neg}^T.$ 
Hence, with probability $\geq 1- P_{\neg, \hat{\cal B}},$ for every $\xi$-aligned hyperrectangle $b^{(\xi)}$ of volume $\xi^N$
contained in $\hat{\cal B}$ there is at least one tree $t$ whose training set contains at least $n$ points from $b^{(\xi)}.$
Since we are assuming that our algorithm uses $\xi$-aligned splits, in each tree, the points from $b^{(\xi)}$ will all be assigned the same leaf $\ell.$  Let $b_{\ell}$ be the hyperrectangle associated to $\ell.$ Since the tree is built so that leaves are pure, the leaf $\ell$ and hence all the points in $b_{\ell}$ are classified as positive. Moreover, since $b_{\ell}$ contains $\geq n$ samples, in the output of the algorithm, there is a hyperrectangle containing $b_{\ell},$ i.e., either $b_{\ell}$ itself or some hyperrectangle that completely contains it.  
Since this holds simultaneously for every $\xi$-aligned hyperrectangle $b^{(\xi)}$ of volume $\xi^N$ contained in $\hat{\cal B}$ it follows that 
${\cal B}^A  \cap {\cal B}^* \supseteq \hat{\cal B},$ whence
$Vol({\cal B}^A  \cap {\cal B}^*) \geq Vol(\hat{\cal B}) \geq (\frac{k-2}{k})^N Vol({\cal B}^*),$ which proves the first inequality in the statement of the theorem. 

For the right inequality, we note that from $b_{\ell}$ we have sampled $\geq n$ points all testing positive. Hence, by 
Corollary \ref{lemma:wilks} with probability at least $1-\delta$, at least a fraction $R$ of  $b_{\ell}$ contains only positive points, i.e, it is part of the positive preimage.
Considering all the boxes returned, we get $Vol(\mathcal{B}^A \cap \mathcal{B}^*) \geq R \sum_i Vol(b_i^A) = R Vol(\mathcal{B}^A)$ from which directly follows $Vol(\mathcal{B}^A \setminus
\mathcal{B}^*) = Vol(\mathcal{B}^A) - Vol(\mathcal{B}^A \cap \mathcal{B}^*) \leq (1-R) Vol(\mathcal{B}^A)$, concluding the proof.
\end{proof}

These theoretical results show that the ensemble of positive leaves produced by \texttt{RF-ProVe} has strong probabilistic guarantees on both purity and coverage. Importantly, since \texttt{RF-ProVe} aggregates the positively classified regions from all $T$ trees, the total covered region $\mathcal{B}$ can only grow larger than in the single-tree case. 
In practice, it is often significantly higher, thanks to the complementary contributions from multiple trees, a phenomenon clearly confirmed by our empirical evaluation.

%====== EMPIRICAL

\section{Empirical Evaluation}

In this section, we investigate whether our new random-forest-inspired method, \texttt{RF-ProVe}, can generate more compact solutions and better scale with both the input dimensionality and the encoding constraints of the problem. We begin our empirical evaluation by analyzing how to set the hyperparameters of \texttt{RF-ProVe} to ensure probabilistic guarantees on both the confidence and the purity of the collected regions. 

\textbf{How to select the hyperparameters?}
In \texttt{RF-ProVe}, two main hyperparameters guide performance and guarantees: the training set size $m$, and the total number of resampling points $n$ used to validate leaf purity. While there is no closed-form rule for selecting $m$, as it depends on input dimension, and desired property to verify, we empirically found that using $m = 20000$ uniformly sampled examples provides a sufficiently dense coverage of the input space to populate the leaf regions of the decision trees across various depths. It also ensures that each tree receives a diverse subset of examples via bootstrapping, preserving both region purity and ensemble diversity. Larger values of $m$ yield diminishing returns while increasing training costs.
The number of total resampling points $n$ is derived from Theorem~\ref{theorem:coverage} and depends on the confidence level $1 - \delta$, the minimum required purity $R$, and the maximum number of candidate regions $|\mathcal{B}|_{max}$, which is dictated by the forest structure. For trees of depth $D$, each can produce up to $2^{D-1}$ pure positive leaves, so a forest with $T$ trees yields $|\mathcal{B}|_{max} = T \cdot 2^{D-1}$. Fig.~\ref{fig:sample_corr} (top) shows that even for $D=11$, achieving up to $1024$ boxes, the total needed resamples stay under $1.5M$ for $\delta = 0.001$ and $R = 0.995$, ensuring a very efficient solution. Crucially, rather than relying on deep trees that risk overfitting, we favor many shallow ones to enhance generalization via randomized partitions. Fig.~\ref{fig:sample_corr} (bottom) shows that even for a fixed extreme maximum number of boxes (e.g., 32000), using depths $D \in [5,7]$ allows for forests with $500$–$2000$ trees. We adopt $D=5$ in all experiments, offering a scalable and expressive partitioning of the input space.

\begin{figure}[!ht]
    \centering
    \includegraphics[width=0.8\linewidth]{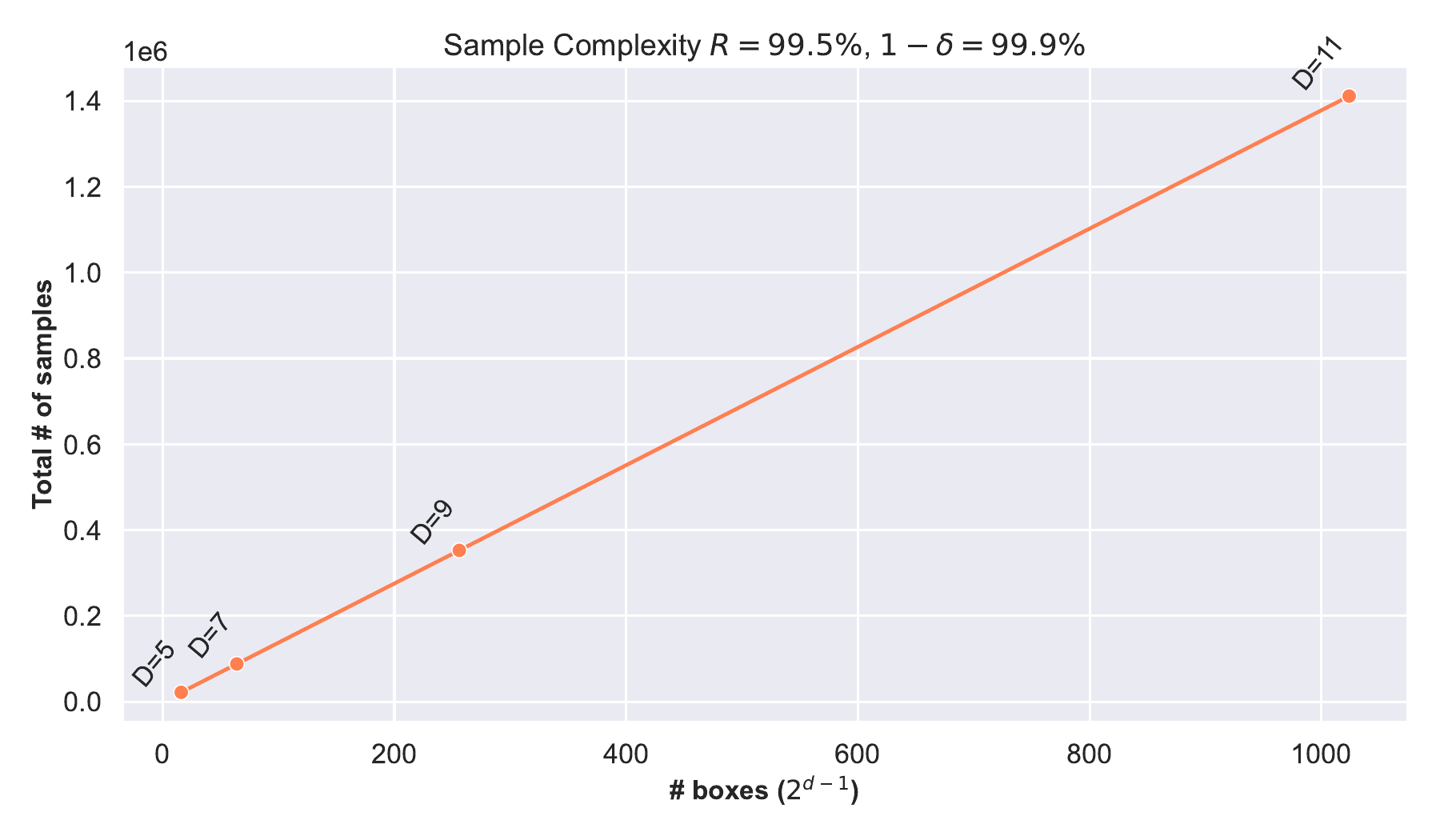}
    \includegraphics[width=0.8\linewidth]{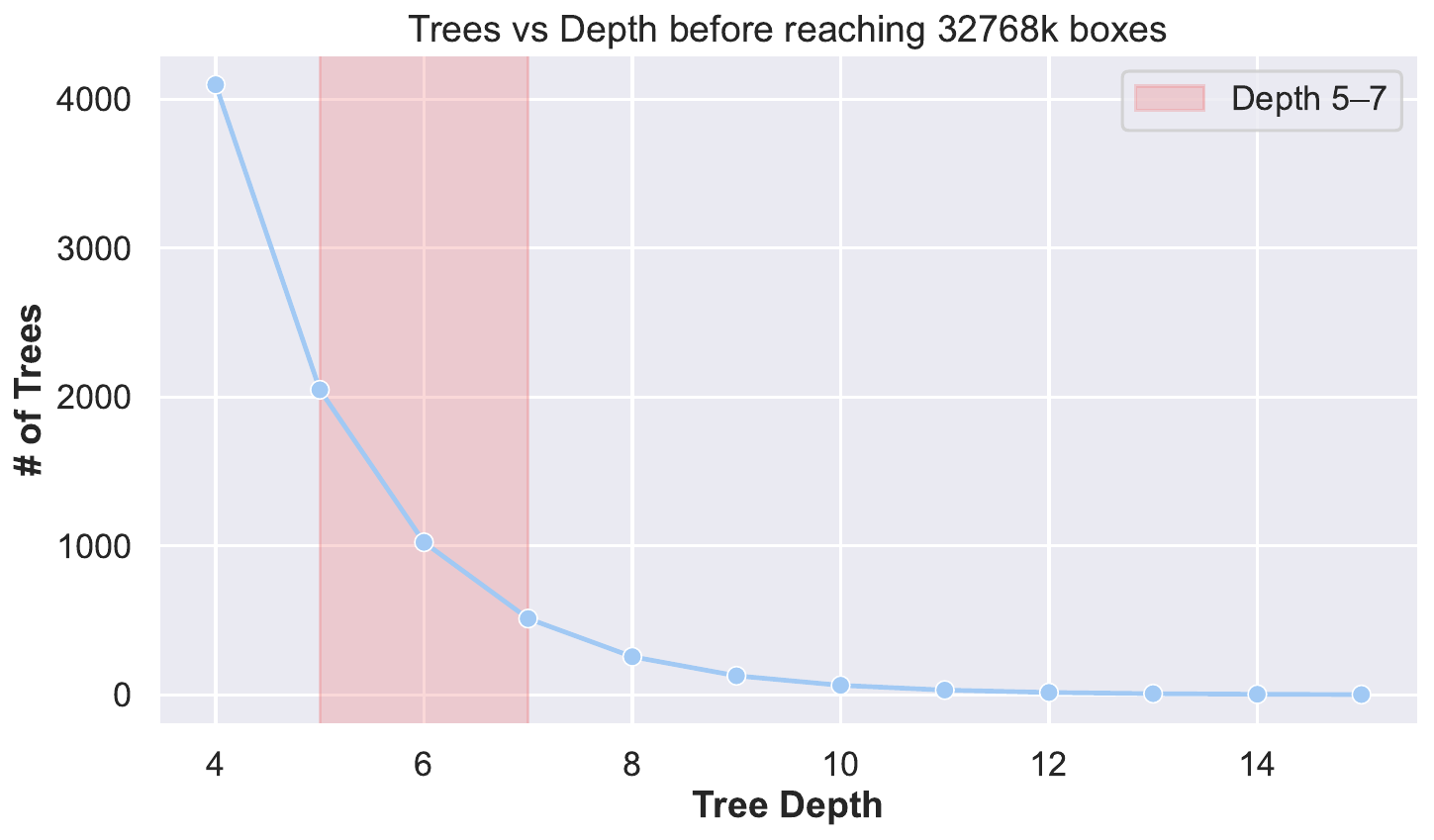}
    \caption{Correlation samples complexity, number of trees, and depth decision trees.}
    \label{fig:sample_corr}
    \vspace{-3mm}
\end{figure}

\begin{table*}[t]
    \tiny
    \centering
    \begin{tabular}{lccccccc}
    
    %\toprule
    \textbf{Method} & \textbf{Task} & \textbf{Property} & \textbf{Config} & \textbf{\#Poly} & \textbf{Coverage} & \textbf{\%error} & \textbf{Time}\\
    \toprule
    Exact & VCAS & $\{y \in \mathbb{R}^9\;|\; \wedge_{i\in [1,8]} y_0 \geq y_i\}$ & as in \cite{matoba2020exact}  & 131 & 100\% & 0\% & 6352.21s\\
    PREMAP & VCAS & $\{y \in \mathbb{R}^9\;|\; \wedge_{i\in [1,8]} y_0 \geq y_i\}$ & as in \cite{matoba2020exact}  & \textbf{15} & \textbf{90.8\%} & 0\% & 12.8s\\
     $\varepsilon$-ProVe & VCAS & $\{y \in \mathbb{R}^9\;|\; \wedge_{i\in [1,8]} y_0 \geq y_i\}$ & as in \cite{matoba2020exact}  & 122 & 90.48\% & 0.02\% & 0.65s\\
     \rowcolor{gray!20}
    \texttt{RF-ProVe} & VCAS & $\{y \in \mathbb{R}^9\;|\; \wedge_{i\in [1,8]} y_0 \geq y_i\}$ & as in \cite{matoba2020exact}  & \textbf{15} & 90.5\% & 0.06\% & 0.3s\\
    \midrule
    PREMAP & Cartpole & $\{y \in \mathbb{R}^2\;|\; y_0 \geq y_1\}$ & $\Dot{\theta} \in [-2,0]$  & 66 & 75.5\% &  0\% &32.37s\\
     $\varepsilon$-ProVe & Cartpole & $\{y \in \mathbb{R}^2\;|\; y_0 \geq y_1\}$ & $\Dot{\theta} \in [-2,0]$  & 72 & 76.47\% & 0.27\% & 2s\\
     \rowcolor{gray!20}
    \texttt{RF-ProVe} & Cartpole & $\{y \in \mathbb{R}^2\;|\; y_0 \geq y_1\}$ & $\Dot{\theta} \in [-2,0]$  & \textbf{22} & \textbf{76.8\%} & 0.3\% & 4.5s\\
    \midrule
     PREMAP & Lunarlander & $\{y \in \mathbb{R}^4\;|\; \wedge_{i\in \{0,2,3\}} y_1 \geq y_i\}$ & $\Dot{v} \in [-4,0]$  & 97 & 75.1\% & 0\%& 85.42s\\
      $\varepsilon$-ProVe & Lunarlander & $\{y \in \mathbb{R}^4\;|\; \wedge_{i\in \{0,2,3\}} y_1 \geq y_i\}$ & $\Dot{v} \in [-4,0]$  & 440 & 76.51\% & 0.5\%& 12.2s\\
     \rowcolor{gray!20}
    \texttt{RF-ProVe} & Lunarlander & $\{y \in \mathbb{R}^4\;|\; \wedge_{i\in \{0,2,3\}} y_1 \geq y_i\}$ & $\Dot{v} \in [-4,0]$  & \textbf{42} & \textbf{75.63\%} & 0.3\% & 59s\\
    \midrule
    % PREMAP & Dubinsrejoin & $\{y \in \mathbb{R}^8\;|\; (\wedge_{i\in [1,3]} y_0 \geq y_i) \bigwedge (\wedge_{i\in [5,7]} y_4 \geq y_i)\}$ & $x_v \in [-0.3,0.3]$  & 677 & 75\% & 0\% & 621.47s\\
    PREMAP & Dubinsrejoin & $\{y \in \mathbb{R}^8\;|\; (\wedge_{i\in [1,3]} y_0 \geq y_i) \bigwedge (\wedge_{i\in [5,7]} y_4 \geq y_i)\}$ & $x_v \in [-0.3,0.3]$  & 1002 & 78.7\% & 0\% & 656.47s\\
    $\varepsilon$-ProVe & Dubinsrejoin & $\{y \in \mathbb{R}^8\;|\; (\wedge_{i\in [1,3]} y_0 \geq y_i) \bigwedge (\wedge_{i\in [5,7]} y_4 \geq y_i)\}$ & $x_v \in [-0.3,0.3]$  & 4929 & 85.02\% & 0.3\% & 260.23s\\
     \rowcolor{gray!20}
    \texttt{RF-ProVe} & Dubinsrejoin & $\{y \in \mathbb{R}^8\;|\; (\wedge_{i\in [1,3]} y_0 \geq y_i) \bigwedge (\wedge_{i\in [5,7]} y_4 \geq y_i)\}$ & $x_v \in [-0.3,0.3]$  & \textbf{136} & \textbf{90.08\%} & 0.3\% & 66s\\
    
    \bottomrule
    \end{tabular}
      \caption{Empirical evaluation results of preimage approximation for reinforcement learning tasks, with Exact \cite{matoba2020exact}, PREMAP \cite{zhang2025premap}, $\varepsilon$-ProVe \cite{eProve} and \texttt{RF-ProVe} in gray proposed in this work.}
      \label{tab: tab_eval}
      \vspace{-2mm}
 \end{table*}
\textbf{Verification experiments} 
We compare \texttt{RF-ProVe} against the Exact \cite{matoba2020exact} solution, provable sound PREMAP \cite{zhang2025premap}, as well as the probabilistic approach $\varepsilon$-ProVe \cite{eProve}. All these approaches compute the preimage using unions of axis-aligned hyperrectangles, making them directly comparable in both representation and output format. In our evaluation, we consider standard verification benchmarks used in \cite{zhang2025premap}, such as the aircraft collision avoidance system (\textit{VCAS}) from \cite{julian2019reachability}, and reinforcement learning (RL) tasks, such as \textit{Cartpole}, \textit{Lunarlander}, and \textit{Dubinsrejoin}.\footnote{We refer the interested readers to \cite{zhang2025premap} for a comprehensive overview of the selected tasks.} Notably, we focus on structured, verification-relevant domains (e.g., \textit{Dubinsrejoin}) where compact preimage bounds are interpretable and actionable. Image datasets like MNIST or CIFAR lack such semantics and are less meaningful for safety analysis. Since methods like PREMAP and $\varepsilon$-ProVe already struggle with \textit{Dubinsrejoin}, higher-dimensional image inputs would add stress without offering additional insight.
To evaluate the quality of the solutions produced by the tested methods, we follow the approach proposed in \cite{zhang2025premap}, using for all approximate methods the same number of samples ($10k$) to estimate the coverage, and define a target \textit{coverage ratio} for each task. 
Given the stochastic nature of the \texttt{RF-ProVe}, results Tab. \ref{tab: tab_eval} including the number of polytopes (\# Poly), the achieved coverage, the percentage of impurity (for probabilistic methods), and the runtime across the tested models, report the average result over 3 random initializations. Moreover, we set a desired confidence in the result of $1-\delta \geq 99.9\%$ (i.e, $\delta=0.001$) and a maximum error in the final solution of $1-R \leq 0.005$ (i.e., $R=0.995$). Our goal is to compute the most compact representation of the preimage region, i.e., using the fewest number of polytopes—while achieving a target level of coverage and ensuring zero, or statistically bounded, impurity. 
All data are collected on an RTX 2070, and an i7-9700k.

\textbf{VCAS task results.} For the first task, we consider the entire set of VCAS models of the benchmark and we set a desired coverage ratio of at least $90\%$ as in \cite{zhang2025premap}. Tab. \ref{tab: tab_eval} reports the mean across all the tested models. As we can notice, the Exact method \cite{matoba2020exact} achieves full coverage but at a prohibitive cost, as it requires over 130 polytopes and takes more than 6300 seconds on average to complete. This highlights the scalability bottleneck of exact methods , which even on simpler instances struggle to scale. Importantly, our \texttt{RF-ProVe} achieves the same number of polytopes as PREMAP \cite{zhang2025premap} (15) while maintaining extremely low impurity (less than 0.1\%) but with an increase of $20\times$ faster runtime, showcasing the power of bootstrapped, data-driven strategies over fixed symbolic solvers.

\textbf{RL task results.} In this experiment, we evaluate preimage approximation methods on neural network controllers across several reinforcement learning tasks. Specifically, we target a coverage of 75\% for \textit{Cartpole} and \textit{Lunarlander}, and 90\% for the more challenging \textit{DubinsRejoin} task \citep{dubinsrejoin}. The Exact method \cite{matoba2020exact} is omitted from this evaluation, as it cannot scale to networks of this size. The results demonstrate the effectiveness of our proposed method. Across all tasks, \texttt{RF-ProVe} consistently matches or exceeds the coverage achieved by existing methods, while requiring significantly fewer polytopes and less computation time. The benefit of our approach is particularly evident in the \textit{DubinsRejoin} task, where PREMAP fails to meet the 90\% coverage target, achieving only 78.7\% coverage despite generating over 1000 polytopes and requiring more than 650 seconds. Similarly, $\varepsilon$-ProVe fails to meet the desired coverage, reaching just 85\% while producing a large number of polytopes before encountering memory issues. In contrast, \texttt{RF-ProVe} attains 90.08\% coverage using just 136 polytopes and 66 seconds, with an impurity of only 0.3\%, crucially below the $1-R=0.5\%$ desired.
This highlights a key strength of our approach: by allowing an infinitesimal error, we can efficiently approximate high-coverage preimages with high confidence, even for complex tasks where exact or provable methods are no longer practical. These results demonstrate the scalability and practical relevance of \texttt{RF-ProVe}, offering a valuable alternative for real-world safety-critical applications where soundness can be slightly relaxed in favor of crucial safety information gains.
\begin{table}[b]
    \tiny
    \centering
    \begin{tabular}{lccccc}
    
    %\toprule
    \textbf{Method} & \textbf{Task} & \textbf{\#Poly} & \textbf{Coverage} & \textbf{\%error} & \textbf{Time}\\
    \toprule
    \texttt{RF-ProVe} & Cartpole & 19 & 75.48\% & 0.39\% & 2.6s\\
     \rowcolor{gray!20}
    \texttt{RF-ProVe} & Cartpole & \textbf{22} & \textbf{76.8\%} & \textbf{0.3\%} & 4.5s\\
    \midrule
    \texttt{RF-ProVe} & Lunarlander & 190 & 76.33\% & 3.54\% & 20s\\
     \rowcolor{gray!20}
    \texttt{RF-ProVe} & Lunarlander & \textbf{42} & \textbf{75.63\%} & \textbf{0.3\%} & 59s\\
    \midrule
    \texttt{RF-ProVe} & Dubinsrejoin  & 308 & 90.26\% & 3.43\% & 39s\\
     \rowcolor{gray!20}
    \texttt{RF-ProVe} & Dubinsrejoin  & \textbf{136} & \textbf{90.08\%} & \textbf{0.3\%} & 66s\\
    \bottomrule
    \end{tabular}
      \caption{Ablation study of preimage approximation for reinforcement learning tasks, with \texttt{RF-ProVe}  without filtering phase (in white) and original (in gray).}
      \label{tab: tab_ablation}
      \vspace{-2mm}
 \end{table}
 
\textbf{Ablation study.} To assess the contribution of our active resampling strategy, we evaluate the performance of \texttt{RF-ProVe} with and without this phase. Specifically, we consider the solution of the method that skips the filtering step and directly returns the pure positive leaves selected by the Gini index from each decision tree, even if the number of positive samples in the leaf is fewer than the one derived theoretically. This isolates the effect of resampling on compactness (number of polytopes), correctness (error rate), and runtime.
Table \ref{tab: tab_ablation} summarizes the results on the RL benchmarks. Across all tasks, active resampling consistently reduces impurity by over an order of magnitude, from $>3\%$ down to less $0.5\%$, while also producing significantly more compact solutions. For instance, in the \textit{LunarLander} task, the number of polytopes drops from $190$ to $42$ with nearly identical coverage. While the resampling step introduces a moderate runtime overhead (roughly $2\times$), the added cost is negligible compared to the error reduction and interpretability gain.
These results highlight that active resampling is crucial to achieving the desired statistical guarantees of \texttt{RF-ProVe}. Without it, the method tends to overfit sparse training data, returning leaf regions that appear pure but actually include a substantial number of non-positive inputs. Hence, we can conclude that the filtering phase effectively corrects this bias by validating each candidate box using a statistically derived number of additional samples, ensuring high-confidence guarantees on region purity. Scalability experiments are reported in the appendix.

%==== DISCUSSION

\section{Discussion}

In this work, we addressed the computational intractability of exact neural network preimage bound computation by proposing a novel probabilistic framework, \texttt{RF-ProVe}. Our approach 
exploits the strength of bootstrap-based and randomized methods to capture complex structures in high-dimensional input spaces, introducing a random forest-inspired method that combines passive learning with active resampling to approximate preimage regions with high-confidence guarantees.
Our novel theoretical results provide strong probabilistic guarantees on region purity and global coverage of the returned solution. Empirically, \texttt{RF-ProVe} significantly produces compact solutions, while maintaining low impurity and high coverage, even on complex verification tasks where existing exact, provable, and probabilistic methods fail to scale. Overall, \texttt{RF-ProVe} represents a promising shift toward scalable, data-driven verification tools that retain strong probabilistic guarantees. Future work may explore its integration with hybrid verification pipelines and extensions to richer geometric representations.

\clearpage
\section*{Acknowledgement}
This work has been supported by PNRR MUR project PE0000013-FAIR.
\bibliography{aaai2026}

\clearpage
\onecolumn
\section*{Appendix}
\section{Scalability experiments}
To evaluate the scalability of \texttt{RF-ProVe}, we conducted experiments on a synthetic dataset with input dimensionalities ranging in $\{2, 5, 7, 10, 15, 20, 30\}$, thus more than doubling the dimensions used in the empirical evaluation in the main paper. For each dimensionality, we run \texttt{RF-ProVe} fixed number of training samples $m$, namely 20000, and resampling points $n=200$, to get a maximum tolerable error rate $1-R=5\%$. In all the experiments, we measure three key metrics: (i) the error rate, representing the fraction of points incorrectly classified as safe, (ii) the number of polytopes and coverage, representing the fraction of the true safe region captured by \texttt{RF-ProVe}, and (iii) the number of trees employed in the verification and their corresponding runtime, representing the computational cost of verification. 

\begin{figure}[h!]
    \centering
    \includegraphics[width=1\linewidth]{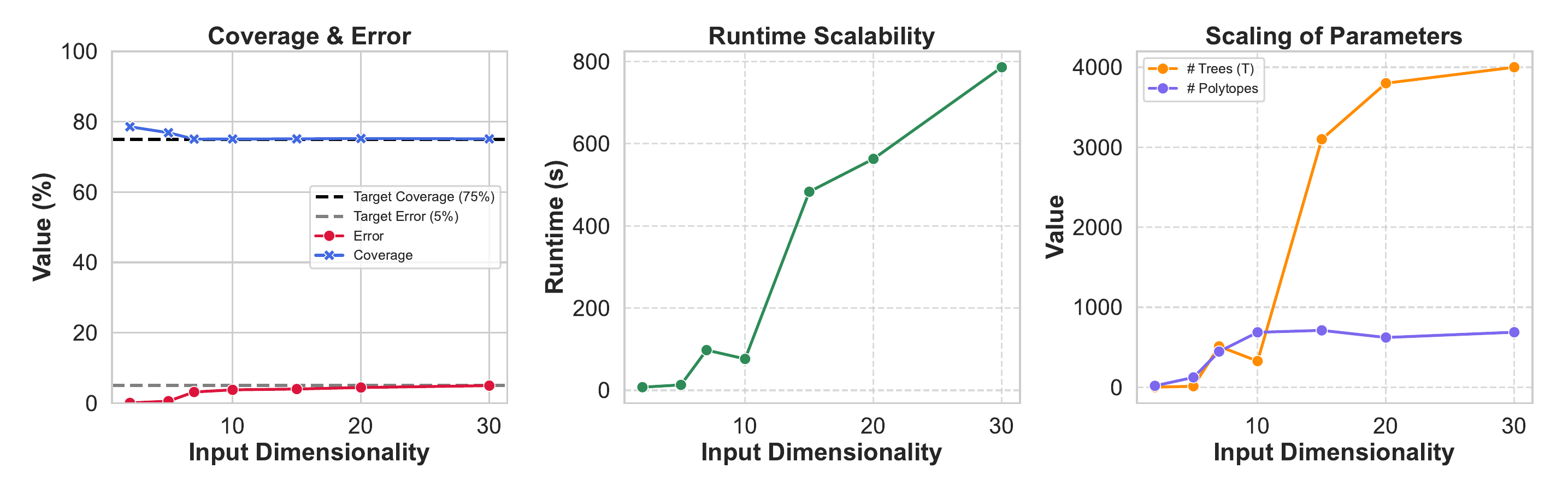}
    \caption{Scalability results on \texttt{RF-ProVe} for growing input dimensionality.}
    \label{fig:scalability}
\end{figure}

The results reported in Figure \ref{fig:scalability}, indicate that both coverage and error remain stable across increasing input dimensions, with coverage consistently above 75\% and error below 5\%. Runtime grows sublinearly, demonstrating that the approach remains computationally feasible even in high-dimensional spaces. Furthermore, the results show that, as the input dimensionality increases, both the number of trees and the number of polytopes returned by RF-ProVe grow approximately logarithmically with dimensionality, demonstrating that the method adapts efficiently to higher-dimensional spaces. This behavior highlights the effectiveness of the proposed solution in maintaining accurate coverage and low error while remaining computationally feasible.

\end{document}